\documentclass{article}

\usepackage{iclr2027_conference,times}

\usepackage{amsmath}
\usepackage{amssymb}
\usepackage{mathtools}
\usepackage{amsthm}

\usepackage{booktabs}
\usepackage{multirow}
\usepackage{array}
\usepackage{graphicx}
\usepackage{microtype}
\usepackage[table]{xcolor}
\usepackage{url}
\usepackage{enumitem}
\usepackage{algorithm}
\usepackage{algpseudocode}
\usepackage{wrapfig}
\usepackage{caption}

\usepackage{pifont}

\usepackage{hyperref}

\definecolor{todo}{RGB}{150,40,40}
\definecolor{bottlegreen}{RGB}{10,150,50}
\definecolor{opsdyellow}{RGB}{250,212,107}
\definecolor{opasdblue}{HTML}{5D8DFE}

\theoremstyle{plain}
\newtheorem{theorem}{Theorem}[section]

\newtheorem{lemma}[theorem]{Lemma}

\theoremstyle{definition}

\theoremstyle{remark}

\title{
Teach Yourself Where to Look:
On-Policy Attention Self-Distillation for Reasoning
}

\author{
\begin{tabular}{@{}l@{}}
\textbf{Safaeid Hossain Arib}\textsuperscript{1,\ding{41}},
\textbf{Rabeya Akter}\textsuperscript{2},
\textbf{Ismam Nur Swapnil}\textsuperscript{1},\\
\textbf{Md. Faiyaz Abdullah Sayeedi}\textsuperscript{3},
\textbf{Tasnim Mohiuddin}\textsuperscript{4,*},
\textbf{Md Mofijul Islam}\textsuperscript{5,*,\(\ddagger\)}\\[0.5em]
{\normalfont
\textsuperscript{1}ACI PLC, Bangladesh
\qquad
\textsuperscript{2}University of Dhaka, Bangladesh}\\
{\normalfont
\textsuperscript{3}BRAC University, Bangladesh
\qquad
\textsuperscript{4}QCRI, Qatar
\qquad
\textsuperscript{5}Amazon GenAI, USA}
\end{tabular}
}

\iclrfinalcopy

\begin{document}

\maketitle


\begingroup
\renewcommand{\thefootnote}{\fnsymbol{footnote}}
\footnotetext[1]{Equal supervision.}
\footnotetext[3]{Work done outside role at Amazon.}
\endgroup

\begingroup
\renewcommand{\thefootnote}{}
\footnotetext{
\ding{41}\ Corresponding author: \texttt{safaeid48@gmail.com}
}
\endgroup



\begin{abstract}

On-policy self-distillation trains reasoning models on their own trajectories using dense token distribution guidance from a privileged teacher with access to a verified solution. This supervision transfers what the teacher predicts without directly transferring where it attends within the preceding context. We introduce On-Policy Attention Self-Distillation (OPASD), which complements token-level supervision with solution-conditioned attention distillation. Because the privileged teacher can attend to verified-solution tokens unavailable to the student, OPASD projects teacher attention onto student-visible positions and renormalizes the resulting distribution before alignment. Across three model sizes and four competition-level mathematics benchmarks, OPASD consistently outperforms token-only OPSD, improving average accuracy by 4.98 to 8.40 percentage points.      OPASD also avoids the response-length inflation and performance degradation observed with token-only distillation, reducing generated rollout tokens by 73.9\% and estimated model compute by 72.6\% while training 1.53× faster. These results show that solution-conditioned attention provides a complementary supervision signal that makes on-policy self-distillation more accurate, stable, and compute-efficient.

\end{abstract}


\section{Introduction}
\label{sec:intro}


Improving language models' reasoning through post-training requires supervision that is informative and aligned with their own generated trajectories. Reinforcement learning with verifiable rewards (RLVR) learns from self-generated responses but typically provides a response-level reward with limited guidance for individual token decisions \citep{shao2024deepseekmath, guo2025deepseek, yu2026dapo}. On-policy distillation (OPD) provides finer-grained supervision by having a teacher evaluate each prefix of a student-generated trajectory and provide a full next-token distribution \citep{agarwal2024policy}. On-policy self-distillation (OPSD) further removes the need for a separately trained teacher by initializing both student and teacher from the same model \citep{zhao2026self}. The student generates a reasoning trajectory from the problem alone, while the teacher evaluates that trajectory with access to a verified solution. OPSD thus provides dense, solution-conditioned token-level supervision along the student's own reasoning process.


Despite this advantage, OPSD transfers privileged guidance only through next-token distributions. Such token-level supervision guides what the student predicts but does not directly guide where it allocates attention across the preceding context~\citep{zhao2026self, yang2026oprd}. This distinction matters in multi-step reasoning, where later steps often depend on constraints and intermediate results established in earlier steps~\citep{chen2025improving, guo2025learning}. The privileged teacher's attention distribution offers a distinct solution-conditioned signal that OPSD leaves unused. This raises the question: \textit{How can on-policy self-distillation teach the student not only what to predict but also where to look?}



Unlike prior attention-alignment work on multimodal grounding or internal self-distillation~\citep{li2026reinforced, liu2026oisd}, extending OPSD with attention-level supervision introduces a challenge absent from next-token distribution alignment. The teacher and student next-token distributions are directly comparable because both are defined over the same vocabulary. But attention distributions are defined over positions in their respective conditioning contexts. The student conditions on the problem and its generated prefix, while the privileged teacher additionally sees the verified solution. The teacher can therefore assign attention to solution tokens absent from the student's context. Direct alignment would require the student to match probability mass on positions it cannot observe, whereas simply discarding those positions leaves the teacher's remaining attention unnormalized. This creates a support mismatch between the teacher's solution-conditioned attention distribution and the student's attention distribution. An effective distillation target must preserve the teacher's solution-conditioned guidance while remaining fully defined over positions visible to the student. 

We address this challenge with On-Policy Attention Self-Distillation (OPASD), which complements token-level OPSD with solution-conditioned attention supervision along the same student-generated trajectory, as illustrated in Figure~\ref{fig:opsd-opasd}. At each reasoning step, OPASD projects the privileged teacher's attention onto the student-visible support by removing solution-only positions and renormalizing the remaining mass. The resulting target remains conditioned on the verified solution but is defined entirely over positions visible to the student. 


Across Qwen3 models from 1.7B to 8B and four competition-level mathematics benchmarks, OPASD consistently outperforms token-only OPSD, improving average accuracy by 4.98 to 8.40 percentage points.  
Importantly, the improvement is accompanied by substantially lower training cost rather than additional computational burden. OPASD reduces generated rollout tokens by 73.9\% and estimated model compute by 72.6\%, while training 1.53× faster despite the additional attention objective. 
OPASD also exhibits more stable training and validation performance, shorter reasoning trajectories, and smaller shifts in epistemic verbalization. Together, these results suggest that solution-conditioned attention provides a complementary supervision signal that improves not only reasoning accuracy but also the stability and efficiency of on-policy self-distillation.








In summary, we make three contributions:
\begin{itemize}
    \item We introduce OPASD, extending privileged on-policy self-distillation with solution-conditioned attention supervision and a student-support projection that resolves the attention-support mismatch between privileged teachers and students.
     \item We show consistent improvements over token-only OPSD across three model scales and four competition-level mathematics benchmarks, with average accuracy gains of 4.98 to 8.40 points.
     \item We show that OPASD substantially improves the efficiency and stability of on-policy training, reducing generated rollout tokens by 73.9\% and estimated compute by 72.6\%, while training 1.53× faster.
\end{itemize}

\begin{figure}[!t]
\centering
\includegraphics[width=0.85\textwidth]{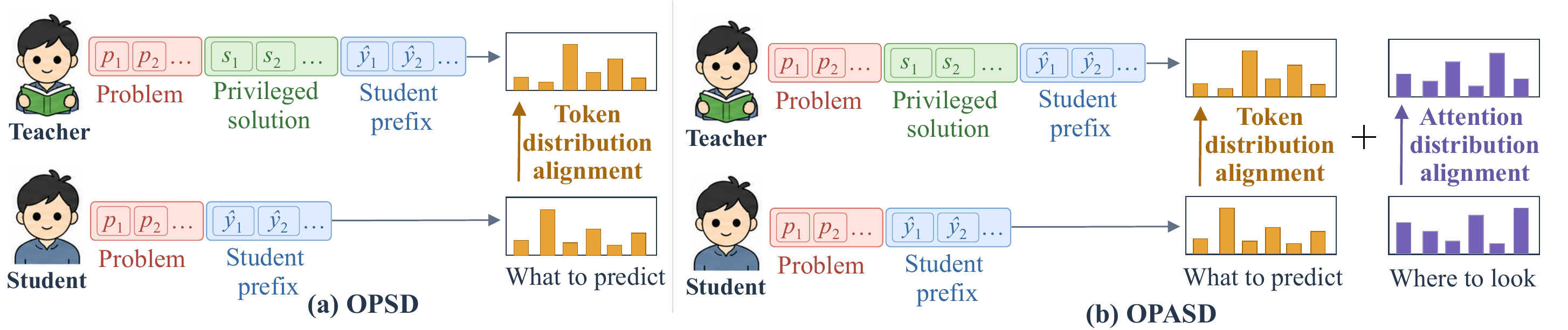}
\caption{Comparison of OPSD and OPASD. (a) OPSD provides token-level supervision through teacher-student next-token distribution alignment. (b) OPASD provides both token- and attention-level supervision through joint next-token and projected attention distribution alignment.}
\label{fig:opsd-opasd}
\end{figure}
\section{Related Work}
\label{sec:related}


\textbf{On-Policy Distillation.}
Conventional knowledge distillation for autoregressive language models trains students on fixed reference or teacher-generated sequences, creating a mismatch between the prefixes used for distillation and those generated by the student at inference time \citep{hinton2015distilling, kim2016sequence, bengio2015scheduled}.
OPD mitigates this mismatch by sampling trajectories from the student and querying the teacher at the resulting student-generated prefixes, providing dense supervision along the student’s own rollout \citep{agarwal2024policy,gu2024minillm,yang2026learning,li2026rethinking,fu2026revisiting,hou2026uni,ko2026scaling,song2026survey}. OPSD extends this principle to self-distillation by using the same model as both student and privileged teacher under different conditioning contexts. The student generates the reasoning trajectory from the problem alone, while the privileged teacher additionally conditions on the verified solution and provides token-level supervision along that same trajectory \citep{zhao2026self}.
Related work explores self-distillation with additional context, demonstrations, environment feedback, and other forms of guidance 
\citep{snell2022learning,qi2025context,hubotter2026reinforcement,shenfeld2026,ye2026policy,penaloza2026privileged,stein2026gates,sang2026policy,ding2026hdpo,li2026unifying,yang2026self,he2026self}.  
Recent studies also examined when on-policy distillation is effective and explored different ways to improve the supervision it provides
\citep{li2026rethinking,kim2026does,harne2026privileged,yang2026learning,stein2026gates,yang2026ogls,jang2026stable,jin2026entropy,jia2026asymmetric,xu2026tip,zhu2026many,li2026rise}.




\textbf{Representation Distillation.} Intermediate-representation distillation complements output matching by aligning teacher and student representations. Early methods match features at selected layers or relations between layers \citep{romero2015fitnets,yim2017gift}, while later work extends intermediate supervision to Transformer language models through hidden-state, attention, and relation-based objectives \citep{sun2019patient,sanh2019distilbert,jiao2020tinybert,sun2020mobilebert,wang2020minilm,wang2021minilmv2,wang2023distill}. These methods generally train on fixed inputs, whereas recent work aligns teacher and student hidden states at selected layers and response positions along student-generated trajectories \citep{yang2026oprd}. This brings representation-level supervision to states visited by the current policy. Unlike hidden-state alignment, OPASD explicitly aligns attention distributions over the preceding context, with the additional challenge that its privileged teacher observes solution tokens unavailable to the student.




\textbf{Attention Distillation.}
Attention has long served as an intermediate target for knowledge distillation, from transferring attention maps in convolutional networks to aligning attention matrices and self-attention relations in Transformers \citep{zagoruyko2016paying, jiao2020tinybert, sun2020mobilebert, wang2020minilm, wang2021minilmv2, wang2023distill, zhao2024no, jin2024align}.
More recent work has applied attention supervision to language reasoning, where teacher attention is used to guide students toward task-relevant information during multi-step reasoning \citep{chen2025improving, guo2025learning}.
Attention-level supervision has also been explored in multimodal post-training, including on-policy attention distillation \citep{kim2026compodistill, li2026reinforced}.
Recent work has further considered internal attention self-distillation, aligning intermediate-layer attention with a detached final-layer teacher over the same reasoning context \citep{liu2026oisd}.
These studies establish attention as a useful source of supervision but do not address attention alignment in privileged on-policy self-distillation, where a solution-conditioned teacher evaluates student-generated reasoning over a different context. OPASD resolves this mismatch by projecting solution-conditioned teacher attention onto student-visible positions before aligning it with student attention alongside the token-level objective.


\section{On-Policy Attention Self-Distillation}
\label{sec:method}

Building on OPSD \citep{zhao2026self}, we extend on-policy self-distillation beyond token-level supervision with an attention-level objective.
Along each student-generated reasoning trajectory, the same privileged teacher provides supervision over both next-token and attention distributions.
Because the teacher additionally observes the reference solution, its attention distribution is defined over positions that are unavailable to the student and therefore cannot be aligned directly with the student's attention distribution.
We resolve this support mismatch by projecting the teacher attention onto the student-visible support and combining the resulting attention supervision with the original token-level objective.

\begin{figure}[!tb]
\centering
\includegraphics[width=0.95\textwidth]{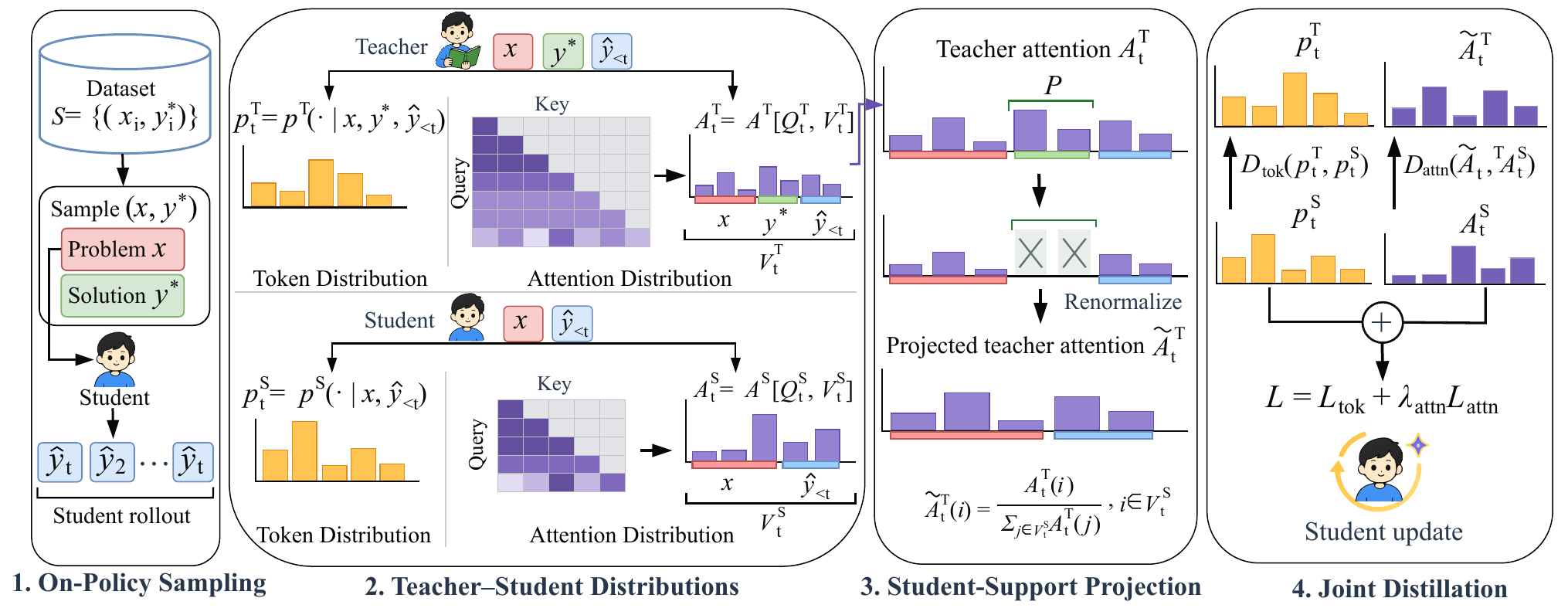}
\caption{Overview of On-Policy Attention Self-Distillation (OPASD). (1) Given a problem \(x\) and its verified solution \(y^\star\), the student samples an on-policy reasoning trajectory \(\hat{y}\) from \(x\) alone. (2) The student and privileged teacher evaluate the same trajectory under their respective conditioning contexts, producing next-token distributions \(p_t^S,p_t^T\) and attention distributions \(A_t^S,A_t^T\). (3) Because the teacher can attend to reference-solution positions unavailable to the student, its attention is restricted to the student-visible support and renormalized to obtain \(\widetilde{A}_t^T\). (4) OPASD jointly aligns the token and projected attention distributions, with gradients propagated only through the student.}
\label{fig:opasd-overview}
\end{figure}

\subsection{On-Policy Self-Distillation}
\label{sec:opsd}

We consider a reasoning dataset
$\mathcal{S}=\{(x_i,y_i^\star)\}_{i=1}^{N}$,
where $x_i$ denotes a problem and $y_i^\star$ its verified reference solution.
Following OPSD, one language model defines a privileged teacher and a student under different conditioning contexts; the teacher uses the initial parameters $\theta_0$, kept frozen throughout training.
The teacher conditions on both the problem and its reference solution, while the student conditions only on the problem, matching the inference-time condition,
\begin{equation}
\label{eq:eq1}
p_T(\cdot \mid x,y^\star)
\triangleq
p_{\theta_0}(\cdot \mid x,y^\star),
\qquad
p_S(\cdot \mid x)
\triangleq
p_\theta(\cdot \mid x).
\end{equation}

For each problem $x$, the student samples an on-policy reasoning trajectory
\begin{equation}
\label{eq:eq2}
\hat{y}
=
(\hat{y}_1,\ldots,\hat{y}_T)
\sim
p_S(\cdot \mid x).
\end{equation}
The privileged teacher evaluates this same trajectory rather than generating its own.
At reasoning step $t$, the two policies produce
\begin{equation}
\label{eq:eq3}
p_t^S
=
p_S(\cdot \mid x,\hat{y}_{<t}),
\qquad
p_t^T
=
p_T(\cdot \mid x,y^\star,\hat{y}_{<t}).
\end{equation}
The token-level distillation objective is
\begin{align}
\label{eq:token_loss}
\mathcal{L}_{\mathrm{tok}}
=
\frac{1}{T}
\sum_{t=1}^{T}
D_{\mathrm{tok}}
\left(
p_t^T,
p_t^S
\right),
\end{align}
where $D_{\mathrm{tok}}$ denotes the token-distribution divergence.
Gradients are propagated only through the student, while the privileged teacher provides a detached target.
We retain this objective and additionally use the privileged teacher for attention-level supervision.

\subsection{Support Mismatch in Privileged Attention}
\label{sec:support_mismatch}

The token distributions $p_t^S$ and $p_t^T$ are directly comparable because both are defined over the same vocabulary.
Attention distributions are instead defined over context positions, and the teacher context additionally contains reference-solution positions that are absent from the student's context.
We average the attention distributions across the $N_h$ heads of the final transformer layer,
\begin{equation}
A^S
=
\frac{1}{N_h}
\sum_{h=1}^{N_h} A_{\mathrm{last},h}^S,
\qquad
A^T
=
\frac{1}{N_h}
\sum_{h=1}^{N_h} A_{\mathrm{last},h}^T.
\label{eq:attention_aggregation}
\end{equation}

For each generated token $\hat{y}_t$, let $Q_t^S$ and $Q_t^T$ denote its corresponding query positions in the student and teacher sequences, and let $V_t^S$ and $V_t^T$ denote their visible supports.
Let $\mathcal{P}$ denote the reference-solution positions available only to the teacher, such that
$V_t^T = V_t^S \cup \mathcal{P}$ and
$V_t^S \cap \mathcal{P} = \varnothing$.
The corresponding attention distributions are
$A_t^S = A^S[Q_t^S,V_t^S]$ and
$A_t^T = A^T[Q_t^T,V_t^T]$.

The privileged teacher can therefore assign attention mass to positions in $\mathcal{P}$ that are unavailable to the student.
Consequently, $A_t^S$ and $A_t^T$ are defined over different supports and cannot be aligned directly.
We resolve this mismatch by projecting the teacher attention onto the student-visible support $V_t^S$.

\begin{algorithm}[t]
\caption{On-Policy Attention Self-Distillation (OPASD)}
\label{alg:opasd}
\begin{algorithmic}[1]

\Require Reasoning dataset
$\mathcal{S}=\{(x_i,y_i^\star)\}_{i=1}^{N}$;
language model $p_\theta$;
token divergence $D_{\mathrm{tok}}$;
attention divergence $D_{\mathrm{attn}}$;
attention-loss coefficient $\lambda_{\mathrm{attn}}$

\State Define $p_S(\cdot\mid x)$ from $p_\theta$ and $p_T(\cdot\mid x,y^\star)$ from the frozen initialization $p_{\theta_0}$

\While{not converged}

    \State Sample a minibatch $\mathcal{B}\subset\mathcal{S}$

    \ForAll{$(x,y^\star)\in\mathcal{B}$}

        \State Sample an on-policy trajectory
        $\hat{y}\sim p_S(\cdot\mid x)$

        \State Evaluate both policies on $\hat{\mathbf{y}}$ to obtain $\{p_t^S,p_t^T,A_t^S,A_t^T\}_{t=1}^{T}$; detach teacher

        \State Compute token-level distillation:
        $\mathcal{L}_{\mathrm{tok}}(x,y^\star)
        \gets
        \frac{1}{T}\sum_{t=1}^{T}
        D_{\mathrm{tok}}\!\left(p_t^T,p_t^S\right)$

        \State Project teacher attention:
$\widetilde{A}_t^T(i) \gets
\frac{A_t^T(i)}
{\sum_{j\in V_t^S} A_t^T(j)},
\quad i\in V_t^S$

        \State Compute attention-level distillation:
        $\mathcal{L}_{\mathrm{attn}}(x,y^\star)
        \gets
        \frac{1}{T}\sum_{t=1}^{T}
        D_{\mathrm{attn}}\!\left(\widetilde{A}_t^T,A_t^S\right)$

    \EndFor

    \State Compute the minibatch loss:
    $\mathcal{L}_{\mathcal{B}}
    \gets
    \frac{1}{|\mathcal{B}|}
    \sum_{(x,y^\star)\in\mathcal{B}}
    \left[
    \mathcal{L}_{\mathrm{tok}}(x,y^\star)
    +
    \lambda_{\mathrm{attn}}
    \mathcal{L}_{\mathrm{attn}}(x,y^\star)
    \right]$

    \State Update the student parameters using
    $\nabla_\theta\mathcal{L}_{\mathcal{B}}$

\EndWhile

\end{algorithmic}
\end{algorithm}

\subsection{Student-Support Projection}
\label{sec:projection}

To obtain a teacher attention target on the same support as the student, we restrict $A_t^T$ to the student-visible positions $V_t^S$ and renormalize the remaining mass (Theorem~\ref{thm:projection}).
For each position $i \in V_t^S$, the projected teacher attention is
\begin{equation}
\widetilde{A}_t^T(i)
=
\frac{A_t^T(i)}
{\sum_{j \in V_t^S} A_t^T(j)}.
\label{eq:projection}
\end{equation}
The resulting distribution $\widetilde{A}_t^T$ is defined over the same support as $A_t^S$.
Although attention assigned to the reference-solution positions is removed, the target remains solution-conditioned because $A_t^T$ is computed under the privileged context $(x,y^\star,\hat{y}_{<t})$.

The projection also preserves the teacher's relative attention over the student-visible positions.
Since every position in $V_t^S$ is rescaled by the same normalization factor, for any $i,j\in V_t^S$ with $A_t^T(j)>0$,
\begin{equation}
\frac{\widetilde{A}_t^T(i)}
     {\widetilde{A}_t^T(j)}
=
\frac{A_t^T(i)}
     {A_t^T(j)}.
\label{eq:ratio_preservation}
\end{equation}

\subsection{Joint Distillation}
\label{sec:objective}

After projection, the student attention $A_t^S$ and the projected teacher attention $\widetilde{A}_t^T$ are defined over the same support. We distill the projected teacher attention using the generalized Jensen-Shannon divergence (JSD). Let $M_t^{(\beta)} =\beta \widetilde{A}_t^T+(1-\beta)A_t^S$, where $\beta\in(0,1)$. The per-step attention objective is
\begin{equation}
\ell_{\mathrm{attn}}(t)
= D_{\mathrm{attn}}
\left(
\widetilde{A}_t^T,A_t^S
\right)
\triangleq
\beta
D_{\mathrm{KL}}
\left(
\widetilde{A}_t^T
\,\Vert\,
M_t^{(\beta)}
\right)
+
(1-\beta)
D_{\mathrm{KL}}
\left(
A_t^S
\,\Vert\,
M_t^{(\beta)}
\right).
\label{eq:attention_loss}
\end{equation}
After rescaling, $\beta\to0$ and $\beta\to1$ recover the forward and reverse KL, which we use as the endpoint variants.
\begin{equation}
\mathcal{L}_{\mathrm{attn}}
=
\frac{1}{T}
\sum_{t=1}^{T}
\ell_{\mathrm{attn}}(t).
\label{eq:attention_loss_total}
\end{equation}
The training objective combines the token-level OPSD loss with the proposed attention-level loss,
\begin{equation}
\mathcal{L}
=
\mathcal{L}_{\mathrm{tok}}
+
\lambda_{\mathrm{attn}}
\mathcal{L}_{\mathrm{attn}}.
\label{eq:total_loss}
\end{equation}
Here $\lambda_{\mathrm{attn}}$ controls the strength of attention supervision. Only the student branch receives gradients; the privileged teacher and reference solution are used only during training. Algorithm~\ref{alg:opasd} summarizes the complete training procedure.

\section{Experiments}
\label{sec:experiments}

\subsection{Experimental Setup}
\label{sec:experimental setup}
\textbf{Models and datasets.}
We experiment with instruction-tuned models from the Qwen3~\citep{team2025qwen3} family: Qwen3-1.7B, Qwen3-4B, Qwen3-8B. For training, we use the mathematical reasoning subset of OpenThoughts~\citep{guha2026openthoughts}, sampling up to 30K problem-solution pairs. We evaluate on competition-level mathematics benchmarks including AIME 2024, AIME 2025, AIME 2026, and HMMT 2025.

\textbf{Baselines.}
We compare against two baselines: (1) \textsc{Base}, the original instruction-tuned checkpoint without post-training; and (2) \textsc{OPSD}~\citep{zhao2026self}, the token-only on-policy self-distillation method. OPSD and our method are trained using the same dataset, model initialization, and training budget.

\textbf{Implementation Details.}
We keep the privileged teacher fixed at the initial checkpoint throughout training. For token-level supervision, we use full-vocabulary logit distillation. For attention-level supervision, we align the student's final-layer attention with the projected teacher attention over the student-visible support. All experiments are conducted on H100 or H200 GPUs using LoRA~\citep{hu2021lora}. Complete training and evaluation configurations are provided in Appendix~\ref{sec:hyperparameters}.

\subsection{Comparison of On-Policy Self-Distillation Methods}
\label{sec:main results}
\begin{table}[!t]
    \centering
    \caption{Comparison of on-policy self-distillation methods across model scales.
Avg@12 accuracy (\%) of the instruction-tuned Base model, token-only OPSD, and OPASD on four competition-level mathematical reasoning benchmarks. OPASD achieves the highest performance across all three Qwen3 model sizes.}
    \label{tab:main_results}
    
    \small
    \setlength{\tabcolsep}{4pt}
    \renewcommand{\arraystretch}{1.0}
    
    \begin{tabular}{l|ccccc}
        \toprule
        \textbf{Method} 
        & \textbf{AIME24} 
        & \textbf{AIME25}
        & \textbf{AIME26}
        & \textbf{HMMT25} 
        & \textbf{Average} \\
        \midrule
        
        \multicolumn{5}{l}{\textit{Qwen3-8B}} \\
        \quad Base  & 60.83 & 48.89 & 51.39 & 29.17 & 47.57 \\
        \quad OPSD          & 59.16 & 51.97 & 53.33 & 33.33 & 49.45 \\
        \quad OPASD      & \textbf{67.50} & \textbf{57.22} & \textbf{61.11} & \textbf{35.83} & \textbf{55.42} \\
        
        \midrule
        \multicolumn{5}{l}{\textit{Qwen3-4B}} \\
        \quad Base & 58.33 & 47.50 & 51.11 & 30.27 & 46.80 \\
        \quad OPSD          & 51.11 & 43.88 & 46.11 & 28.33 & 42.36 \\
        \quad OPASD      & \textbf{63.33} & \textbf{50.27} & \textbf{55.56} & \textbf{33.89} & \textbf{50.76} \\
        
        \midrule
        \multicolumn{5}{l}{\textit{Qwen3-1.7B}} \\
        \quad Base  & 33.61 & 30.28 & 31.94 & 19.17 & 28.75 \\
        \quad OPSD  & 36.70 & 28.33 & 32.78 & 18.33 & 29.04 \\
        \quad OPASD & \textbf{43.05} & \textbf{35.27} & \textbf{34.44} & \textbf{23.33} & \textbf{34.02} \\

        \bottomrule
    \end{tabular}
\end{table}

Table~\ref{tab:main_results} compares OPASD with the instruction-tuned base model and token-only OPSD across three Qwen3 model sizes. OPASD achieves the highest average performance at every scale and outperforms OPSD on all four benchmarks. The average score rises from $29.04$ to $34.02$ on Qwen3-1.7B, from $42.36$ to $50.76$ on Qwen3-4B, and from $49.45$ to $55.42$ on Qwen3-8B, giving gains of $4.98$, $8.40$, and $5.97$ points, respectively. OPASD also surpasses the base model at every scale. These consistent gains show that attention-level supervision complements the token-level objective across model scales and benchmark difficulty levels.


We further evaluate OPASD across sampling budgets, multiple training seeds, and test-time generation budgets in Appendices~\ref{sec:pass_at_k_dynamics}, \ref{sec:random_seeds}, and \ref{sec:generation_budget}, respectively. These analyses show that OPASD improves both Pass@1 and Pass@12, remains stable across training seeds, and consistently outperforms OPSD across generation budgets.


\subsection{Ablation Studies and Analysis}
\label{sec:ablation}







We examine three design choices in OPASD: (1) the divergence used for attention distillation (Section~\ref{sec:attention_divergence}),  (2) the weight of attention supervision (Section~\ref{sec:attention_weight}), and (3) the depth of attention distillation (Section~\ref{sec:ablation_layers}). Additional experiments evaluate how teacher-only attention is handled (Appendix~\ref{sec:privileged_attention}), the token-loss clipping strategy (Appendix~\ref{sec:ablation_clipping}), and the individual contributions of token- and attention-level supervision (Appendix~\ref{sec:objective_components}). We then present a qualitative comparison of OPSD and OPASD reasoning trajectories (Section~\ref{sec:qualitative_analysis}), analyze their performance and behavioral dynamics (Section~\ref{sec:training_dynamics}), and compare their efficiency (Section~\ref{sec:training_efficiency}).


\subsubsection{Attention Divergence Objective}
\label{sec:attention_divergence}

\begin{wraptable}{r}{0.52\textwidth}
\centering
\caption{Effect of the attention divergence objective. JSD achieves the highest average accuracy across the three objectives.}
\label{tab:divergence_comparison}
\small
\setlength{\tabcolsep}{1pt}
\renewcommand{\arraystretch}{1.0}
\begin{tabular}{lccccc}
\toprule
\textbf{Objective} & \textbf{AIME24} & \textbf{AIME25} &
\textbf{AIME26} & \textbf{HMMT25} & \textbf{Avg} \\
\midrule
FKL & 39.72 & \textbf{33.05} & 33.05 & \textbf{21.11} & 31.73 \\
RKL & 40.83 & 32.22 & \textbf{35.83} & 20.56 & 32.36 \\
JSD        & \textbf{41.94} & 32.50 & \textbf{35.83} & \textbf{21.11} & \textbf{32.85} \\
\bottomrule
\end{tabular}
\end{wraptable}

We compare forward KL (FKL) ($\widetilde A_t^T \,\|\, A_t^S)$, reverse KL (RKL) ($A_t^S \,\|\, \widetilde A_t^T)$, and JSD  ($\beta = \,0.5$) for attention distillation on Qwen3-1.7B, while keeping all other settings fixed. 
As shown in Table~\ref{tab:divergence_comparison}, JSD achieves the highest average accuracy of 32.85, compared with 32.36 for reverse KL and 31.73 for forward KL.
JSD achieves this advantage through a symmetric and bounded objective that balances the mode-covering behavior of forward KL with the mode-seeking behavior of reverse KL. This balance preserves the broader context emphasized by the teacher while retaining focus on its most relevant positions.





\subsubsection{Strength of Attention Supervision}
\label{sec:attention_weight}

\setlength{\intextsep}{6pt}
\setlength{\columnsep}{10pt}
\begin{wraptable}{r}{0.52\textwidth}
\centering
\caption{Effect of attention loss weight. With the token-loss coefficient fixed at $1$, $\lambda_{\mathrm{attn}}=0.5$ yields the highest average performance.}
\label{tab:lambda_attn_ablation}
\small
\setlength{\tabcolsep}{2pt}
\renewcommand{\arraystretch}{1.0}
\begin{tabular}{lccccc}
\toprule
$\boldsymbol{\lambda_{\mathrm{attn}}}$ &
\textbf{AIME24} & \textbf{AIME25} & \textbf{AIME26} &
\textbf{HMMT25} & \textbf{Avg} \\
\midrule
0.25 & \textbf{43.33} & 34.16 & 35.56 & 21.94 & 33.75 \\
0.5  & 43.05 & \textbf{35.27} & 34.44 & \textbf{23.33} & \textbf{34.02} \\
1.0  & 41.94 & 32.50 & \textbf{35.83} & 21.11 & 32.85 \\
2.0  & 39.44 & 32.28 & 34.44 & 21.11 & 31.82 \\
\bottomrule
\end{tabular}
\end{wraptable}

We vary the attention-loss weight $\lambda_{\mathrm{attn}}$ on Qwen3-1.7B while keeping the token-loss coefficient fixed at 1. As shown in Table~\ref{tab:lambda_attn_ablation}, the highest average score is $34.02$ at $\lambda_{\mathrm{attn}}=0.5$, while weights of $1.0$ and $2.0$ yield $32.85$ and $31.82$, respectively. This suggests that attention supervision is most effective as a complementary signal, with larger weights placing excessive emphasis on matching the teacher's attention and reducing the influence of the token-level objective.




\subsubsection{Attention Distillation Depth}
\label{sec:ablation_layers}


\begin{wraptable}{r}{0.52\textwidth}
\centering
\caption{\textbf{Effect of attention distillation depth.}
Using only the final transformer layer achieves the highest average performance compared with using the final two or three layers.}
\label{tab:attention_layers}
\small
\setlength{\tabcolsep}{1.5pt}
\renewcommand{\arraystretch}{1.0}
\begin{tabular}{lccccc}
\toprule
\textbf{Layers} & \textbf{AIME24} & \textbf{AIME25} &
\textbf{AIME26} & \textbf{HMMT25} & \textbf{Avg} \\
\midrule
1  & \textbf{43.05} & \textbf{35.27} & \textbf{34.44} & \textbf{23.33} & \textbf{34.02} \\
2 & 41.13 & 34.57 & 33.36 & 21.11 & 32.54 \\
3 & 39.44 & 30.83 & 33.88 & 20.00 & 31.04 \\
\bottomrule
\end{tabular}
\end{wraptable}

We ablate the number of final transformer layers used for attention distillation on Qwen3-1.7B. For multi-layer variants, we compute a separate loss from the head-averaged attention at each layer and average the losses. As shown in Table~\ref{tab:attention_layers}, using only the final layer achieves the highest average score of $34.02$, compared with $32.54$ and $31.04$ for the final two and three layers. Giving earlier-layer attention equal weight shifts some supervision toward broader intermediate patterns. Distilling only the final layer works better because its attention acts on representations refined by the preceding layers and is more closely tied to the model’s next-token prediction.






\subsubsection{Qualitative Analysis}
\label{sec:qualitative_analysis}

Figure~\ref{fig:reasoning} compares the base model, OPSD, and OPASD on the same number-theory problem. The base model commits to an incomplete solution, while OPSD explores additional possibilities but evaluates them incorrectly. OPASD carries this broader analysis through consistently and reaches the correct solution. OPASD’s attention supervision provides additional guidance on earlier calculations, helping the student evaluate later candidates more consistently than token-level supervision alone.
The probability and geometry examples in Appendix~\ref{sec:add_qualitative} illustrate a similar distinction.

\begin{figure}[!tb]
\centering
\includegraphics[width=0.85\textwidth]{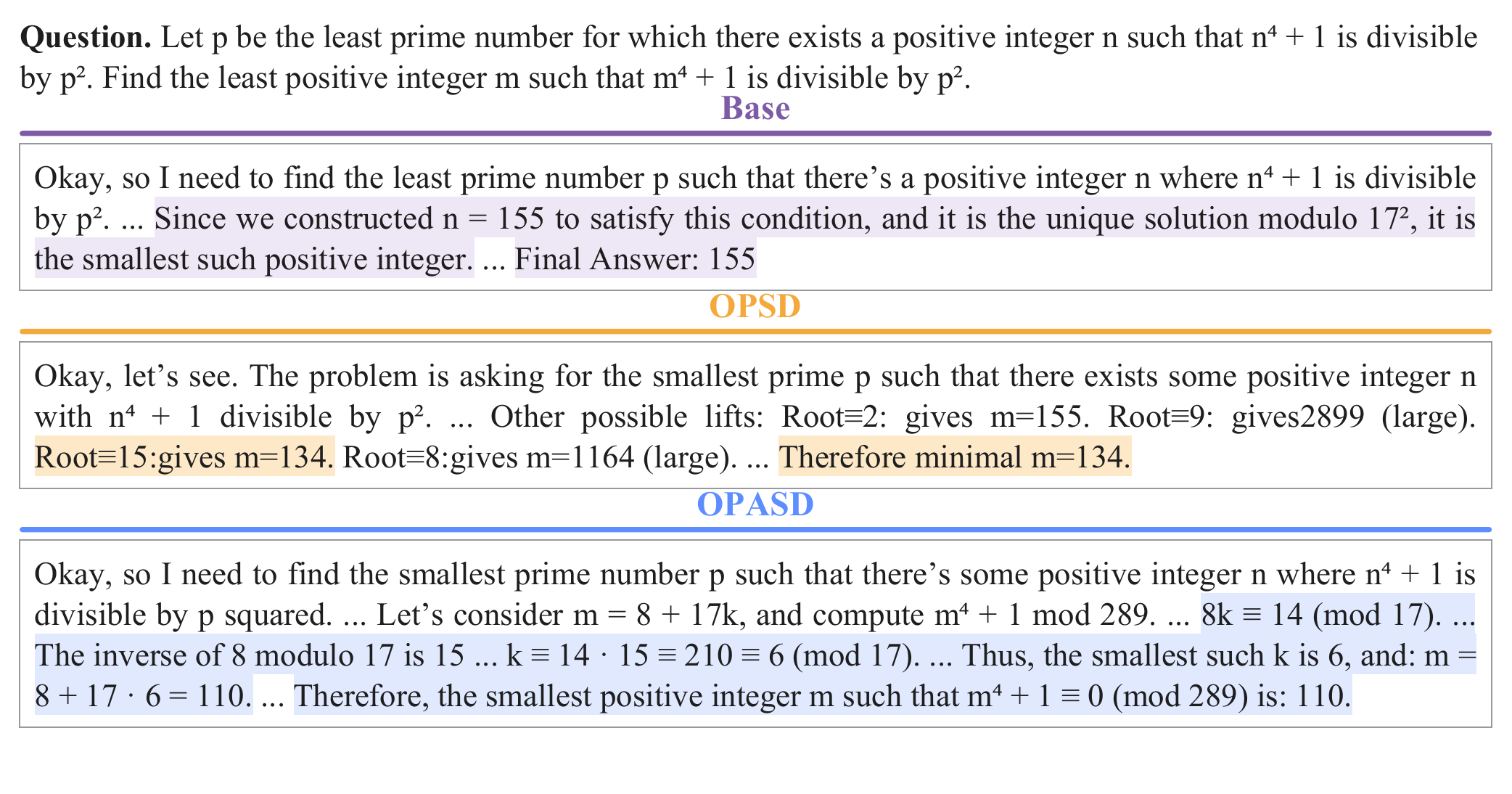}
\caption{Qualitative comparison of reasoning across the base model, OPSD, and OPASD on an AIME24 number-theory problem. The base model overlooks valid alternatives, OPSD considers more candidates but evaluates one incorrectly, and OPASD identifies the correct minimum. Highlighted excerpts mark the decisive steps.}
\label{fig:reasoning}
\end{figure}
\subsubsection{Performance and Behavioral Dynamics}
\label{sec:training_dynamics}


Figure~\ref{fig:training_dynamics} compares OPSD and OPASD in training and validation accuracy, response length, and epistemic-token usage.

\begin{figure}[!b]
\centering
\includegraphics[width=0.95\textwidth]{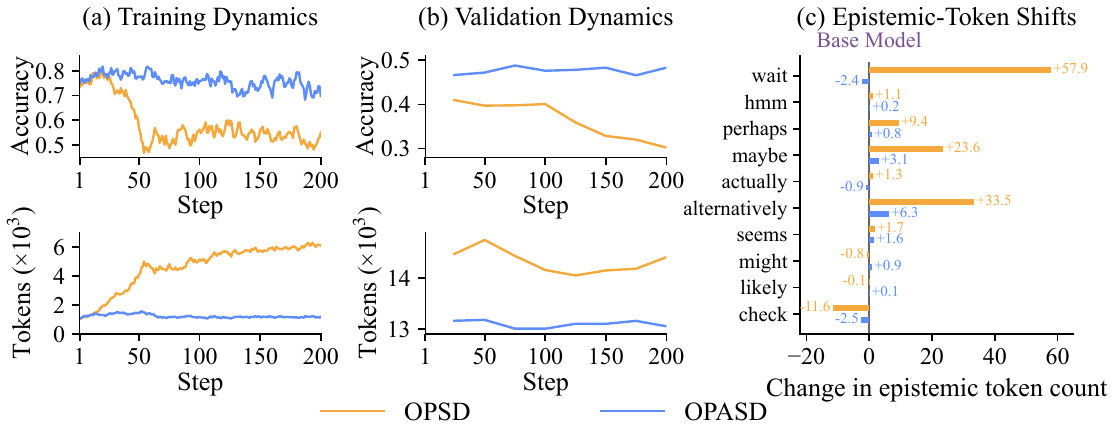}
\caption{Training dynamics and epistemic-token shifts on Qwen3-4B.
OPASD maintains stable training and validation performance with concise responses and limited epistemic-token shifts, whereas OPSD exhibits declining accuracy, response-length inflation, and substantially greater use of uncertainty and reconsideration markers.}
\label{fig:training_dynamics}
\end{figure}



Figure~\ref{fig:training_dynamics}(a) shows that OPASD maintains rollout accuracy between $0.70$ and $0.80$, with response lengths remaining between $1.0$K and $1.6$K tokens. In contrast, OPSD's accuracy declines to $0.50$--$0.60$ as its response length grows from $1.1$K to $6.1$K tokens. Figure~\ref{fig:training_dynamics}(b) shows that this contrast extends to validation. Across the recorded checkpoints, OPASD's Avg@12 accuracy varies by only $2.22$ points, whereas OPSD's accuracy spans $10.76$ points and declines steadily after approximately $100$ steps. OPASD also maintains consistently shorter validation responses. Because OPSD optimizes token-level teacher-student agreement without an explicit outcome signal, it does not directly distinguish continuation that improves the final answer from continuation that merely extends the reasoning process. OPASD adds attention supervision that guides the student toward context emphasized by the solution-conditioned teacher. This contextual signal helps preserve relevant information and limit repetitive continuation as the student progressively learns the teacher's attention distribution (Appendix~\ref{sec:attention_alignment_dynamics}).

Figure~\ref{fig:training_dynamics}(c) compares changes in the mean number of epistemic tokens (expressions of uncertainty, reconsideration, or self-verification)~\citep{kim2026does} per response from the base model to the final OPSD and OPASD checkpoints across all four benchmarks. Across the ten markers, the mean count increases by $70.9\%$ for OPSD but only $4.4\%$ for OPASD. The difference remains after normalizing for response length and within both correct and incorrect groups, with OPSD using \emph{wait}, \emph{alternatively}, and \emph{maybe} markers more often than OPASD (Appendix~\ref{sec:epistemic_dynamics}). Together with OPSD's increasing response length and declining validation accuracy, this result suggests that OPSD produces more repeated uncertainty and reconsideration without improving its reasoning performance. This is because OPSD aligns the next-token distribution without directly guiding which earlier steps the student should revisit to resolve uncertainty. OPASD additionally aligns the student's attention distribution with that of the solution-conditioned teacher, guiding its use of earlier context during reconsideration.

\subsubsection{Efficiency}
\label{sec:training_efficiency}
\setlength{\intextsep}{6pt}
\setlength{\columnsep}{10pt}
\begin{wrapfigure}{r}{0.50\textwidth}
    \centering
    \includegraphics[width=\linewidth]{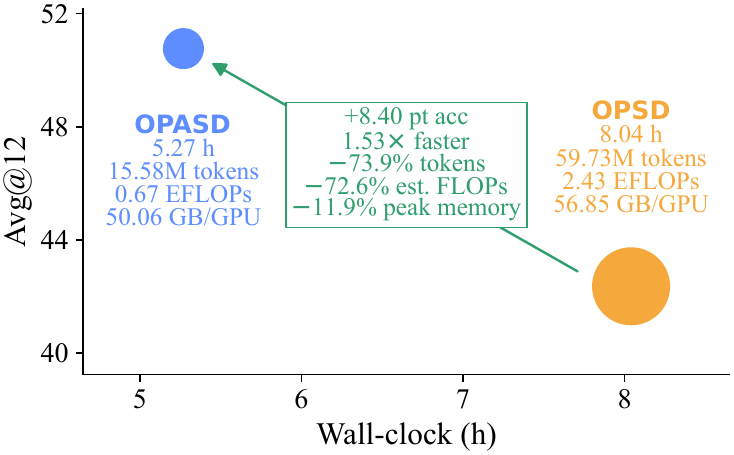}
    \caption{OPASD is Pareto-dominant over OPSD with higher Avg@12 and lower wall-clock time, estimated model compute, generated response tokens, and peak GPU memory. Estimated compute is reported in EFLOPs ($10^{18}$ FLOPs). Bubble area is proportional to the total number of response tokens generated.}
    \label{fig:training_efficiency}
\end{wrapfigure}


Figure~\ref{fig:training_efficiency} compares OPSD and OPASD over $200$ Qwen3-4B training steps on four H100 GPUs. Compared with OPSD, OPASD improves average Avg@12 across the four benchmarks from $42.36$ to $50.76$ and reduces wall-clock time from $8.04$ to $5.27$ hours, a $1.53\times$ speedup. OPASD generates $15.58$ million response tokens compared with $59.73$ million for OPSD, a $73.9\%$ reduction, while also reducing estimated model compute by $72.6\%$ and peak allocated GPU memory by $11.9\%$. Although attention supervision increases actor-update cost from $0.07$ to $0.12$ ms per processed token, OPASD avoids OPSD's response-length growth and substantially reduces the total number of tokens generated and processed. OPASD therefore Pareto-dominates OPSD in accuracy, wall-clock time, generated tokens, and peak memory, with lower estimated model compute as well. We further examine token-learning efficiency in Appendix~\ref{sec:token_learning_efficiency}.

\section{Conclusion}
\label{sec:conclusion}


OPSD distills the privileged teacher’s next-token distribution but leaves its solution-conditioned attention distribution unused. Transferring this signal requires resolving a support mismatch because the teacher can attend to verified-solution positions the student cannot observe. We introduce OPASD to address this mismatch by projecting and renormalizing teacher attention over student-visible positions before distilling it alongside token distributions. Across three model scales and four mathematics benchmarks, OPASD consistently improves accuracy over token-only OPSD. Further analyses show more stable training and validation accuracy, shorter responses, smaller shifts in epistemic-token usage, and lower training cost. Qualitative examples complement these aggregate findings by illustrating more consistent use of intermediate results. Together, these findings show that attention supervision complements token-level self-distillation. 

\textbf{Limitations.} OPASD outperforms OPSD across three Qwen3 model sizes, but its evaluation with other model families would provide a comprehensive assessment. We keep the privileged teacher fixed throughout training. Whether attention supervision offers similar benefits under other teacher-update strategies remains to be studied.



\clearpage

\section*{AI Use Statement}
We used generative AI tools to obtain feedback on experimental design and to assist with the analysis of experimental findings. These tools also supported literature searches, analysis code, figure preparation, and the drafting and editing of parts of the manuscript. The authors reviewed all AI-assisted material and take responsibility for the methods, results, claims, code, and figures in this paper.
\section*{Ethics Statement}

OPASD studies mathematical reasoning using an existing training dataset and established benchmarks. It does not involve human subjects or collect new personal data. Although we evaluate our method on mathematical reasoning, it could be applied in other domains and would inherit risks associated with the underlying models and training data. Models trained with OPASD should be assessed for reliability and misuse before deployment.
\section*{Reproducibility Statement}
Section~3 describes the OPASD training objective and how the privileged teacher's attention is projected onto positions visible to the student. Section~4 reports the models, training data, baselines, evaluation benchmarks, and ablation studies used to assess the method. We will publicly release the code and experiment configurations to enable reproduction of our results.


\bibliography{1.main}
\bibliographystyle{iclr2027_conference}


\appendix

\clearpage

\section{appendix}
\label{sec:appendix}

The appendix provides additional analyses, ablations, implementation
details, and theoretical results supporting the main paper.
\section{Performance Across Sampling Budgets}
\label{sec:pass_at_k_dynamics}

We examine whether OPASD improves the reliability of individual responses or primarily benefits from repeated sampling. Figure~\ref{fig:pass_at_k_dynamics} reports Pass@1 and Pass@12 throughout the Qwen3-4B training run on each benchmark.

\textbf{OPASD produces more reliable individual responses.}
The Pass@1 curves show a consistent separation between the two methods. OPASD maintains higher accuracy across all four benchmarks, whereas OPSD gradually loses performance during later training. At the final checkpoint, OPASD achieves an average Pass@1 of $48.33$, compared with $30.21$ for OPSD. OPASD's attention objective provides later steps with additional guidance about which earlier positions the solution-conditioned teacher emphasizes. Its advantage at Pass@1 therefore reflects stronger individual responses rather than a dependence on repeated attempts.


\textbf{Repeated sampling narrows but does not remove the advantage.}
Pass@12 is more variable because allowing multiple attempts can recover problems that either method solves inconsistently. Nevertheless, OPASD improves the final average Pass@12 from $56.51$ to $63.82$. The gains are clear on AIME24, AIME25, and HMMT25, while the two methods perform comparably on AIME26. The smaller difference under Pass@12 indicates that additional samples partially compensate for the lower single-response accuracy of OPSD. OPASD still retains an advantage because the same attention-supervised model generates each attempt, so its stronger single-response reliability remains relevant when the sampling budget increases.

\begin{figure}[!b]
\centering
\includegraphics[width=\textwidth]{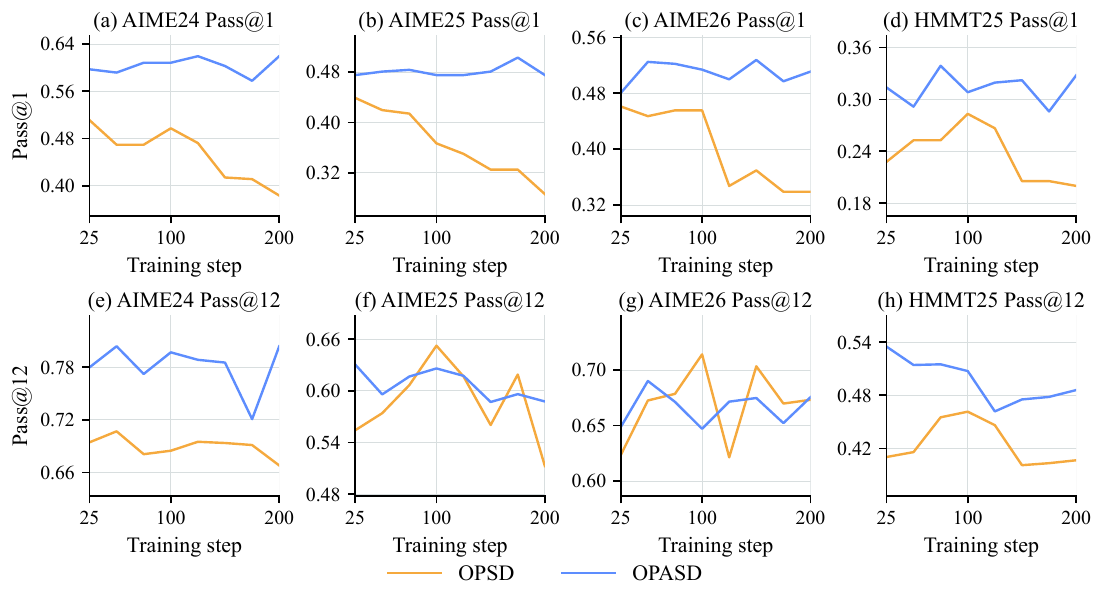}
\caption{Pass@1 and Pass@12 dynamics of OPSD and OPASD over $200$ training steps on Qwen3-4B. OPASD consistently achieves higher Pass@1 across all four benchmarks and retains an overall advantage under Pass@12. Curves show the evaluations recorded for every $25$ training steps.}
\label{fig:pass_at_k_dynamics}
\end{figure}

\section{Token Learning Efficiency}
\label{sec:token_learning_efficiency}

We examine whether OPASD obtains stronger validation performance from the same generation budget. Figure~\ref{fig:token_efficiency} reports Avg@12 accuracy against the cumulative response tokens generated by the training rollouts throughout Qwen3-4B training on each benchmark. Over the complete $200$-step training runs, OPASD generates $15.58$ million response tokens compared with $59.73$ million for OPSD, reducing token usage by $73.9\%$ while consistently achieving higher validation accuracy across all four benchmarks. These results show that OPASD uses the generated-token budget more effectively by maintaining shorter and more productive reasoning trajectories.
This is because OPSD's response-length growth spends much of its budget on repeated uncertainty and reconsideration (Section~\ref{sec:training_dynamics}), which adds tokens without improving answers. OPASD's attention supervision keeps the student focused on relevant context, which is why it needs fewer tokens to reason better.

\begin{figure*}[t]
    \centering
    \includegraphics[width=\textwidth]{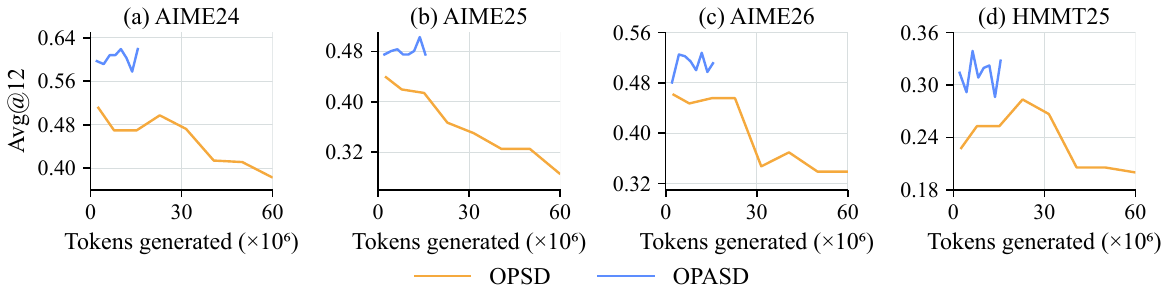}

    \caption{Token-learning efficiency of OPSD and OPASD over $200$ training steps on Qwen3-4B. OPASD consistently achieves higher accuracy across all four benchmarks. The curves report validation accuracy at $25$ step intervals.}
    \label{fig:token_efficiency}
\end{figure*}
\section{Epistemic Verbalization Dynamics}
\label{sec:epistemic_dynamics}

Epistemic verbalization refers to expressions of uncertainty, reconsideration, or self-verification within a reasoning trace. Following~\cite{kim2026does}, we measure ten epistemic markers: \emph{wait}, \emph{hmm}, \emph{perhaps}, \emph{maybe}, \emph{actually}, \emph{alternatively}, \emph{seems}, \emph{might}, \emph{likely}, and \emph{check}. We treat these markers as a behavioral proxy for epistemic verbalization rather than a direct measure of the model's internal uncertainty. Tracking these markers alongside response length and accuracy helps reveal whether training changes the model's reasoning style rather than merely the number of generated tokens. We report marker occurrences per $1{,}000$ response tokens to account for differences in response length.

\textbf{OPASD remains stable while OPSD exhibits increasing epistemic verbalization.}
Figure~\ref{fig:epistemic_density_dynamics} reports epistemic-marker density in the periodic validation responses. Across the four benchmarks, OPASD remains within a narrow range of $12$-$14$ markers per $1{,}000$ tokens throughout training. In contrast, OPSD ranges from  $14$ to $23$ markers and generally increases as training progresses. Figure~\ref{fig:training_dynamics} shows that OPSD's responses also grow longer as its validation accuracy declines, while OPASD remains stable on both measures. For OPSD, increasing reconsideration therefore accompanies more generations without better final answers. OPASD instead supervises how the student attends to earlier reasoning alongside next-token predictions, and maintains shorter responses and higher accuracy.


\begin{figure}[!b]
    \centering
    \includegraphics[width=\columnwidth]{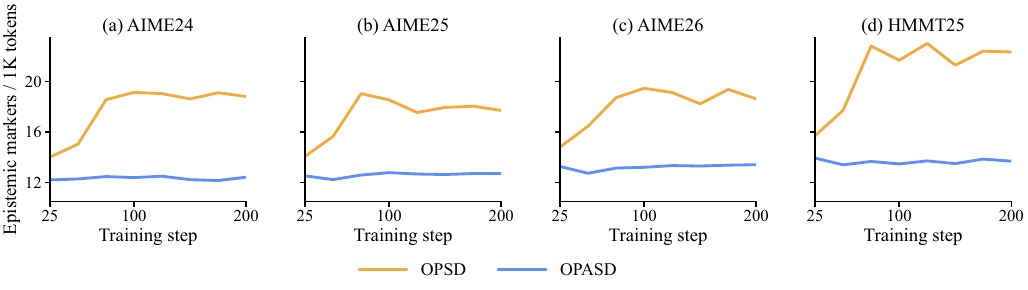}
    \caption{Epistemic-marker density throughout training.
Marker frequency per $1{,}000$ tokens in periodic Qwen3-4B validation responses. OPASD maintains stable epistemic verbalization across all four benchmarks, whereas OPSD exhibits a substantial increase as training progresses.}
    \label{fig:epistemic_density_dynamics}
\end{figure}

\textbf{OPASD maintains lower epistemic-marker density than OPSD regardless of response correctness.} 
Figure~\ref{fig:epistemic_density_correctness} separates the final-checkpoint responses according to correctness. For correct responses, OPSD produces $16.4$ epistemic markers per $1{,}000$ tokens, whereas OPASD produces $12.1$. For incorrect responses, OPSD produces $20.3$ markers, compared with $13.7$ for OPASD. 
Incorrect responses therefore contain more epistemic verbalization under both methods, suggesting that frequent reconsideration is associated with unresolved reasoning. However, OPASD consistently produces fewer markers than OPSD in both groups.  This gap within each group shows that OPASD’s lower marker density is not simply a consequence of its higher accuracy. Its attention supervision guides the use of earlier reasoning, supporting more focused continuations regardless of the final outcome.


\begin{figure}[!t]
    \centering
    \includegraphics[width=0.45\columnwidth]{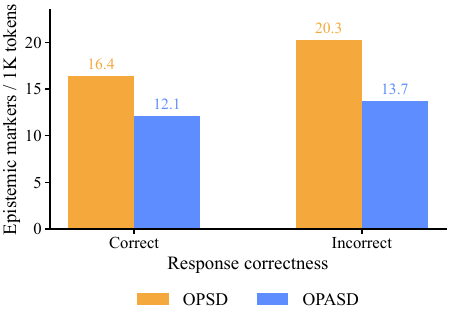}
    \caption{Epistemic-marker density by response correctness.
Marker frequency per $1{,}000$ tokens at the final Qwen3-4B checkpoint, pooled across all four benchmarks. OPASD produces fewer markers than OPSD for both correct and incorrect responses. This pattern is consistent with OPASD avoiding the excessive reconsideration observed in OPSD.}
    \label{fig:epistemic_density_correctness}
\end{figure}

\textbf{OPASD reduces repeated uncertainty while retaining explicit verification.}
Figure~\ref{fig:epistemic_marker_composition} separates the overall density into its individual markers at the final checkpoint. OPSD uses \emph{wait}, \emph{maybe}, \emph{alternatively}, and \emph{perhaps} more frequently than OPASD. The largest difference appears for \emph{wait}, with $7.42$ occurrences per $1{,}000$ tokens for OPSD compared with $3.58$ for OPASD. In contrast, OPASD uses \emph{check} more frequently, with $1.21$ occurrences compared with $0.46$ for OPSD. 
OPASD thus uses fewer markers of uncertainty and alternative search without uniformly reducing epistemic language.
Its attention objective provides guidance over earlier positions to use when reconsidering a step, supporting a more focused response rather than an extended search through alternatives. This matters because the change accompanies shorter responses and higher final-answer accuracy, not merely fewer epistemic markers.


\begin{figure}[!htbp]
    \centering
    \includegraphics[width=0.64\columnwidth]{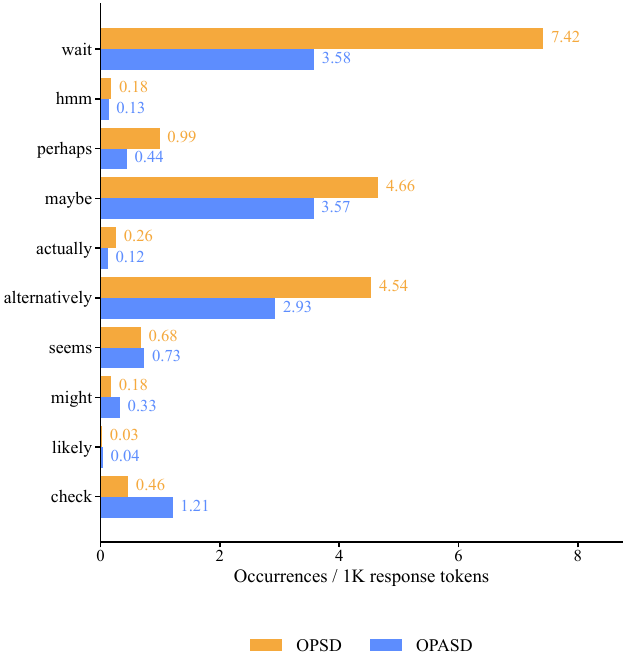}
    \caption{Composition of epistemic verbalization at the final checkpoint.
    Frequency of each epistemic marker per $1{,}000$ response tokens, pooled across all four benchmarks. OPSD's higher overall density is driven primarily by \emph{wait}, \emph{maybe}, \emph{alternatively}, and \emph{perhaps}, whereas OPASD uses \emph{check} more frequently.}
    \label{fig:epistemic_marker_composition}
\end{figure}

\section{Attention Alignment Dynamics}
\label{sec:attention_alignment_dynamics}

We examine whether attention alignment remains effective as the student-generated trajectories evolve during training. Figure~\ref{fig:attention_loss} tracks the JSD between student attention and projected solution-conditioned teacher attention on Qwen3-4B. The loss decreases rapidly during early training and continues to decline thereafter, falling from an average of $0.0055$ over the first 20 steps to $0.0029$ over the final 20 steps.
Because the teacher's target is recomputed for each new trajectory, the decline shows that OPASD keeps aligning attention on the reasoning it currently generates. This sustained alignment allows the teacher's guidance over earlier information to remain relevant as the student's reasoning changes throughout training.


\begin{figure}[!b]
    \centering
    \includegraphics[width=0.45\columnwidth]{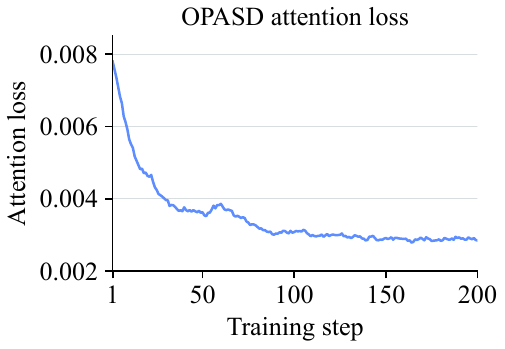}
    \caption{Attention-alignment dynamics of OPASD on Qwen3-4B.
    The divergence between student attention and projected solution-conditioned teacher attention decreases throughout training. The curve is an exponential moving average with smoothing factor $0.7$.}
    \label{fig:attention_loss}
\end{figure}

\section{Robustness across random seeds}
\label{sec:random_seeds}

To assess sensitivity to training randomness, we conduct two additional Qwen3-4B runs with seeds 1 and 123, complementing the original run with seed 42. All runs use the same training data, hyperparameters, compute budget, and experimental configuration. Table~\ref{tab:random_seed_robustness} reports the results across seeds. OPASD achieves an average score of $50.88 \pm 0.14$ (mean $\pm$ standard deviation), with a 95\% confidence interval (CI) of $[50.53, 51.24]$. The small variation across the tested seeds indicates that OPASD's performance is stable with respect to training randomness.

\begin{table}[!b]
\centering
\caption{Robustness of OPASD across random seeds.
Avg@12 accuracy (\%) across three Qwen3-4B training runs with seeds 1, 42, and 123. All runs use the same experimental configuration. The final rows report the mean, sample standard deviation, and 95\% confidence interval across 3 seeds.}
\label{tab:random_seed_robustness}
\small
\setlength{\tabcolsep}{2pt}
\renewcommand{\arraystretch}{1}

\begin{tabular}{lccccc}
\toprule
\textbf{Seed / Statistic}
& \textbf{AIME24}
& \textbf{AIME25}
& \textbf{AIME26}
& \textbf{HMMT25}
& \textbf{Average} \\
\midrule
Seed 1   & 63.27 & 50.83 & 55.33 & 34.72 & 51.04 \\
Seed 42  & 63.33 & 50.27 & 55.56 & 33.89 & 50.76 \\
Seed 123 & 63.54 & 50.63 & 54.87 & 34.36 & 50.85  \\
\midrule
Mean     & 63.38 & 50.58 & 55.25 & 34.32 & 50.88 \\
Std      & 0.14  & 0.28  & 0.35  & 0.42  & 0.14  \\
95\% CI  & $\pm 0.35$ & $\pm 0.70$ & $\pm 0.87$ & $\pm 1.03$ & $\pm 0.36$ \\
\bottomrule
\end{tabular}
\end{table}
\section{Test-Time Generation Budget}
\label{sec:generation_budget}

We examine whether the advantage of OPASD persists across different test-time generation budgets. Both methods are trained with an $8$K-token response cap and evaluated with maximum generation lengths of $16$K, $28$K, and $38$K tokens. As shown in Figure~\ref{fig:generation_budget}, increasing the generation budget improves the average performance of both methods. OPASD outperforms OPSD on every benchmark at all three budgets, with average gains of $8.40$, $6.25$, and $5.81$ points at $16$K, $28$K, and $38$K, respectively. These results show that the advantage of OPASD remains consistent across test-time generation budgets. OPASD stays ahead at every budget because its attention supervision teaches the student to focus on the context emphasized by the solution-conditioned teacher, helping it reason better.

\begin{figure}[!htbp]
\centering
\includegraphics[width=\textwidth]{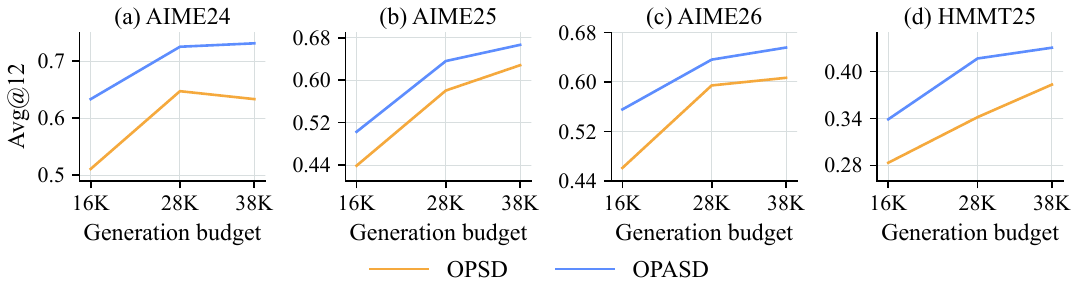}
\caption{Avg@12 performance of OPSD and OPASD across test-time generation budgets on Qwen3-4B.
OPASD consistently outperforms OPSD across all four benchmarks at $16$K, $28$K, and $38$K tokens, while both methods generally benefit from larger generation budgets.}
\label{fig:generation_budget}
\end{figure}

\section{Handling Privileged Attention}
\label{sec:privileged_attention}

We compare three ways to handle teacher attention on positions unavailable to the student using Qwen3-1.7B in Table~\ref{tab:gt_attention_handling}. Student-support projection restricts teacher attention to student-visible positions and renormalizes it. Redistribution transfers the excluded mass to shared prompt positions according to the teacher's attention, while shuffled redistribution randomly changes which shared positions receive that mass. Student-support projection achieves the highest average score of $34.02$, compared with $31.39$ for redistribution and $32.15$ for shuffled redistribution. This advantage is consistent with projection preserving the teacher's relative attention among positions the student can use, whereas redistribution changes those relative weights and shuffling adds arbitrary variation. Theorem~\ref{thm:projection} shows that projection is the unique visible-support distribution closest to the full teacher attention under the attention divergence.


\begin{table}[!b]
\centering
\caption{Handling teacher-only attention mass. Student-support projection achieves a higher average score than redistribution or shuffled redistribution.}
\label{tab:gt_attention_handling}
\small
\setlength{\tabcolsep}{2pt}
\renewcommand{\arraystretch}{1.0}

\begin{tabular}{lccccc}
\toprule
\textbf{Method} & \textbf{AIME24} & \textbf{AIME25} & \textbf{AIME26} & \textbf{HMMT25} & \textbf{Average} \\
\midrule
Redistribution & 40.00 & 31.38 & \textbf{34.44} & 19.72 & 31.39 \\
Shuffled Redistribution        & 42.22 & 31.67 & 33.89 & 20.83 & 32.15 \\
Student-Support Projection           & \textbf{43.05} & \textbf{35.27} & \textbf{34.44} & \textbf{23.33} & \textbf{34.02} \\

\bottomrule
\end{tabular}
\end{table}

\section{Token-Loss Clipping}
\label{sec:ablation_clipping}

We compare three clipping strategies for the token-distillation loss on Qwen3-1.7B in Table~\ref{tab:clipping_comparison}. Point-wise clipping operates on individual vocabulary-level contributions, whereas per-token clipping is applied to the complete KL divergence at each generated position. Using the same threshold of $0.05$, per-token clipping achieves the highest average score of $32.85$, compared with an average score of $32.22$ without clipping and $30.83$ with point-wise clipping. This is because point-wise clipping removes the teacher's strongest token preferences at each position, weakening the token signal that the attention objective complements. Per-token clipping instead preserves the full signal and only limits positions with very large divergence, keeping the two objectives better balanced.

\begin{table}[!t]
\centering
\caption{Effect of token-loss clipping. Per-token clipping achieves the highest average performance compared with point-wise clipping and no clipping.}
\label{tab:clipping_comparison}
\small
\setlength{\tabcolsep}{2pt}
\renewcommand{\arraystretch}{1.0}

\begin{tabular}{lccccc}
\toprule
\textbf{Clipping Strategy} & \textbf{AIME24} & \textbf{AIME25} & \textbf{AIME26} & \textbf{HMMT25} & \textbf{Average} \\
\midrule
Point-wise Clipping & 40.00 & 28.61 & 33.61 & \textbf{21.11} & 30.83 \\
Per-token Clipping  & \textbf{41.94} & 32.50 & \textbf{35.83} & \textbf{21.11} & \textbf{32.85} \\
No Clipping  & 40.00 & \textbf{33.33} & 34.72 & 20.83 & 32.22 \\
\bottomrule
\end{tabular}
\end{table}







\section{Comparison of Distillation Objectives}
\label{sec:objective_components}

Table~\ref{tab:objective_components} compares token-only OPSD against attention-only OPASD and full OPASD on Qwen3-1.7B. Attention-only distillation achieves an average accuracy of $30.97$ and exceeds token-only OPSD by $1.93$ points. Full OPASD combines token- and attention-level distillation and achieves the highest average accuracy of $34.02$. This represents an improvement of $3.05$ points over attention-only distillation and $4.98$ points over token-only OPSD.
Token distillation aligns the student’s next-token distribution with the teacher’s, while attention distillation aligns its attention over student-visible positions with the projected teacher target. Together, these signals supervise both next-token predictions and attention to earlier positions, thus yielding higher accuracy than either objective alone.


\begin{table}[!tb]
\centering
\caption{Comparison of distillation objectives on Qwen3-1.7B.
Full OPASD achieves the highest average accuracy and outperforms both token-only OPSD and attention-only distillation.}
\label{tab:objective_components}
\small
\setlength{\tabcolsep}{2pt}
\renewcommand{\arraystretch}{1.0}

\begin{tabular}{lccccc}
\toprule
\textbf{Training Objective} &
\textbf{AIME24} & \textbf{AIME25} & \textbf{AIME26} &
\textbf{HMMT25} & \textbf{Average} \\
\midrule
Token only (OPSD)
& 36.70 & 28.33 & 32.78 & 18.33 & 29.04 \\
Attention only
& 40.56 & 30.27 & 33.61 & 19.44 & 30.97 \\
Token + Attention (OPASD)
& \textbf{43.05} & \textbf{35.27} & \textbf{34.44} &
\textbf{23.33} & \textbf{34.02} \\
\bottomrule
\end{tabular}
\end{table}
\section{Controlling for Token-Loss Clipping}
\label{sec:matched_token_clipping}

The original OPSD and OPASD configurations use different token-loss clipping strategies. To test whether this difference accounts for OPASD's improvement, we repeat token-only OPSD with the per-token clipping used by OPASD. As shown in Table~\ref{tab:matched_attention_ablation}, this control scores $29.09$ on average, close to the original OPSD score of $29.04$. With per-token clipping held fixed, OPASD scores $34.02$ and outperforms the token-only control on all four benchmarks. The advantage therefore persists when clipping is matched. This is because per-token clipping does not change the token-level signal enough on its own to affect performance, and its benefit appears only when it balances the token loss against the attention objective. The improvement of OPASD over OPSD therefore comes from attention supervision rather than from the clipping strategy.

\begin{table}[!b]
\centering
\caption{Effect of attention supervision on Qwen3-1.7B with per-token clipping held fixed. The original OPSD configuration is included for reference.}
\label{tab:matched_attention_ablation}
\small
\setlength{\tabcolsep}{2pt}
\renewcommand{\arraystretch}{1.0}
\begin{tabular}{lccccc}
\toprule
Method & AIME24 & AIME25 & AIME26 & HMMT25 & Average \\
\midrule
OPSD (point-wise clipping) & 36.70 & 28.33 & 32.78 & 18.33 & 29.04 \\
OPSD (per-token clipping)  & 35.83 & 29.56 & 32.80 & 18.16 & 29.09 \\
OPASD (per-token clipping) & \textbf{43.05} & \textbf{35.27} & \textbf{34.44} & \textbf{23.33} & \textbf{34.02} \\
\bottomrule
\end{tabular}
\end{table}

\section{Additional Qualitative Results}
\label{sec:add_qualitative}

Figures~\ref{fig:aime25_reasoning} and~\ref{fig:aime26_reasoning} compare the instruction-tuned base model, token-only OPSD, and OPASD on an AIME25 probability problem and an AIME26 geometry problem. In the probability example, the base model treats the specific divisors $81$ and $25$ as necessary for the LCM condition, overlooking other divisors that supply the same prime powers. OPSD identifies the exponent conditions but miscounts their overlap, whereas OPASD accounts for the shared divisor and reaches the correct answer. In the geometry example, the base model's coordinates make the required reflex vertex convex, while OPSD's revised coordinates place $B$, $C$, and $D$ on a line. OPASD preserves the nonconvex configuration and uses the resulting area constraint to count the valid side lengths. In both cases, the distinction is whether the constraints established during the setup are carried through to the final calculation.
OPASD’s attention supervision guides the student toward the earlier constraints needed to complete the final calculation.


\begin{figure}[!t]
\centering
\includegraphics[width=0.90\textwidth]{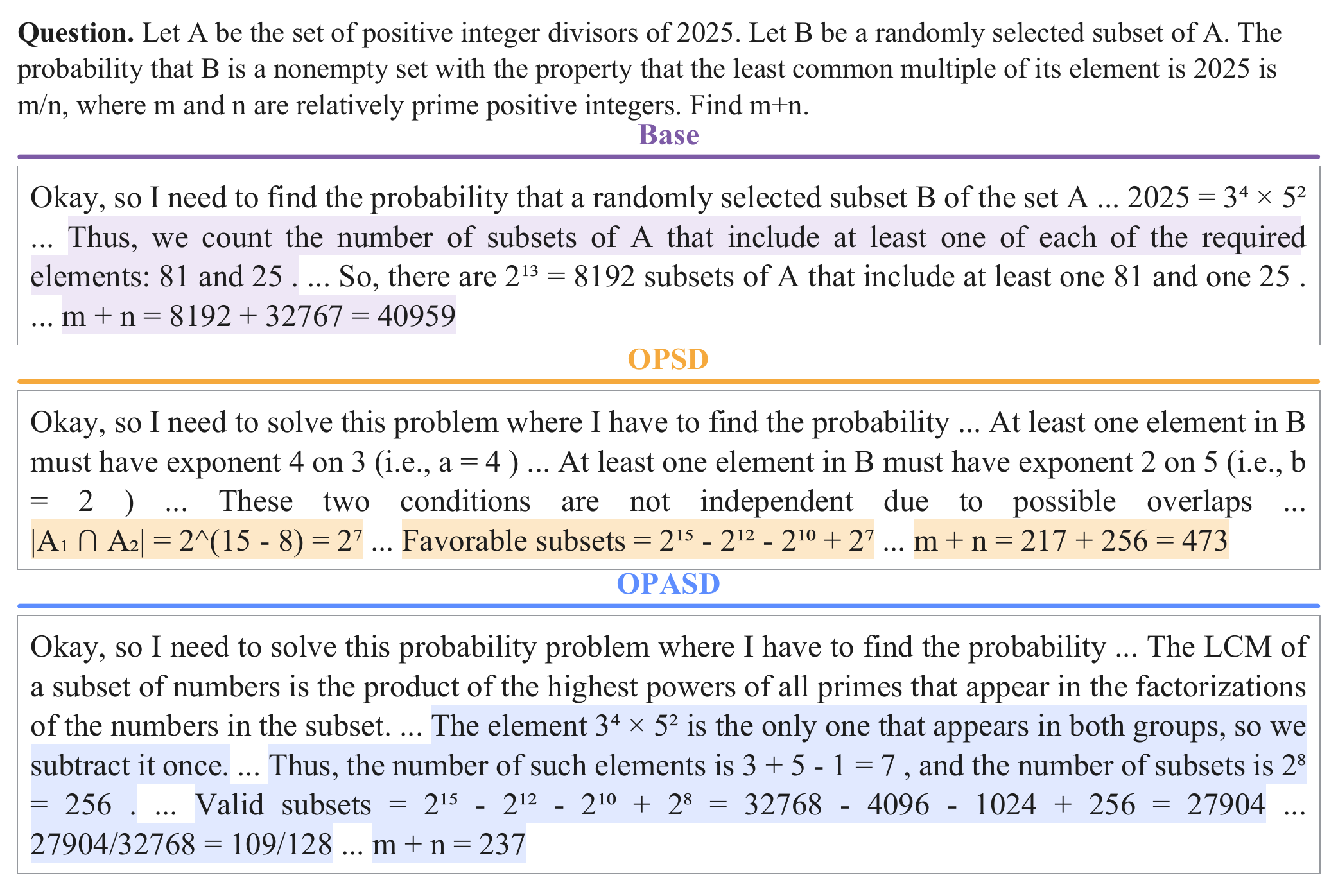}
\caption{Reasoning on an AIME25 probability problem. The base model uses an incorrect counting condition, OPSD miscounts the overlap in inclusion-exclusion, and OPASD reaches the correct answer. Highlighted excerpts mark the decisive steps.}
\label{fig:aime25_reasoning}
\end{figure}

\begin{figure}[!htbp]
\centering
\includegraphics[width=0.90\textwidth]{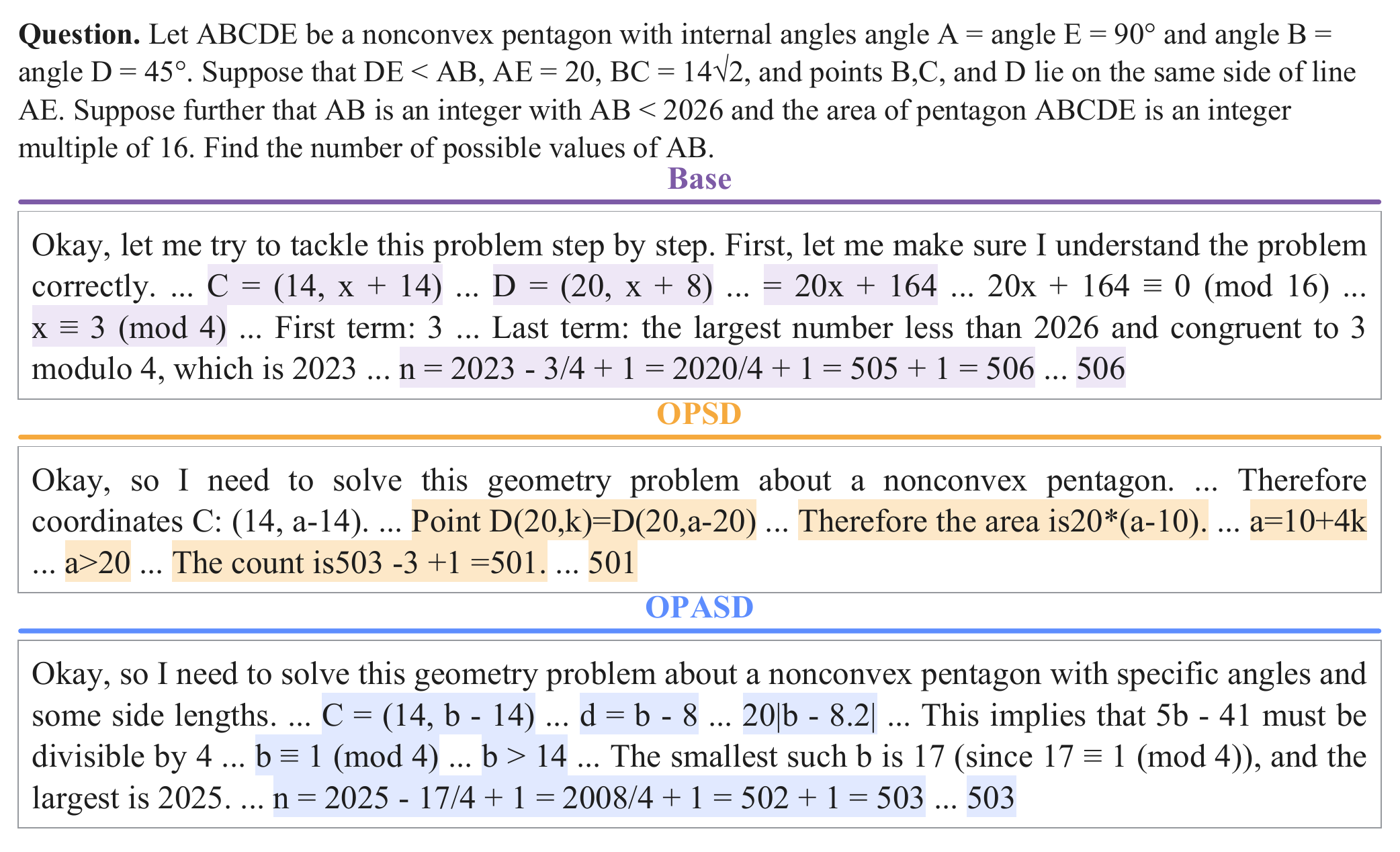}
\caption{Reasoning on an AIME26 geometry problem with an integer-area constraint. The base model and OPSD derive incorrect constraints from their coordinate constructions, while OPASD reaches the correct count. Highlighted excerpts mark the decisive steps.}
\label{fig:aime26_reasoning}
\end{figure}
\section{Hyperparameters}
\label{sec:hyperparameters}
We report the training and evaluation configurations for our OPSD and OPASD experiments in Tables~\ref{tab:training_hyperparameters} and \ref{tab:evaluation_hyperparameters}, respectively. For both methods, we use a Thinking-Mode-off student and a Thinking-Mode-on teacher. OPASD uses the same token-level distillation objective as OPSD and additionally applies attention distillation with $\lambda_{\mathrm{attn}}=0.5$.

\begin{table}[!b]
    \centering
    \caption{Training hyperparameters for OPSD and OPASD.}
    \label{tab:training_hyperparameters}
    \small
    \setlength{\tabcolsep}{2pt}
    \renewcommand{\arraystretch}{1.0}
    \begin{tabular}{lcc}
        \toprule
        \textbf{Parameter} & \textbf{OPSD} & \textbf{OPASD} \\
        \midrule
        Training data & \multicolumn{2}{c}{\shortstack{OpenThoughts-Math-30K (OPSD split)}} \\
        Teacher & \multicolumn{2}{c}{\shortstack{Frozen teacher, conditioned on the reference solution}} \\
        Student thinking mode & Disabled & Disabled \\
        Teacher thinking mode & Enabled & Enabled \\
        LoRA rank $r$ & $64$ & $64$ \\
        LoRA scaling $\alpha$ & $128$ & $128$ \\
        LoRA target modules & \multicolumn{2}{c}{\shortstack{
            \texttt{q\_proj}, \texttt{k\_proj}, \texttt{v\_proj}, \texttt{o\_proj},\\
            \texttt{gate\_proj}, \texttt{up\_proj}, \texttt{down\_proj}}} \\
        Optimizer & AdamW & AdamW \\
        Learning rate & $5\times10^{-6}$ & $5\times10^{-6}$ \\
        Effective batch size & $64$ & $64$ \\
        Training steps & $200$ & $200$ \\
        Maximum prompt length & $2{,}048$ & $2{,}048$ \\
        Maximum completion length & $8{,}192$ & $8{,}192$ \\
        Temperature & $1.1$ & $1.1$ \\
        Top-$p$ & $0.95$  & $0.95$  \\
        Top-$k$ & 20 & 20\\
        \midrule
        Token-loss weight $\lambda_{\mathrm{tok}}$ & $1.0$ & $1.0$ \\
        Token divergence & Forward KL  & Forward KL \\
        Token-loss clipping & Point-wise & Per-token \\
        Clipping threshold & 0.05 & 0.05      \\
        Distillation temperature & $1.1$ & $1.1$ \\
        Vocabulary support & Full vocabulary & Full vocabulary \\
        \midrule
        Attention-loss weight $\lambda_{\mathrm{attn}}$ & - & $0.5$ \\
        Attention divergence & - & Jensen-Shannon \\
        \bottomrule
    \end{tabular}
    
\end{table}

\begin{table}[!t]
    \centering
    \caption{Evaluation hyperparameters.}
    \label{tab:evaluation_hyperparameters}
    \small
    \setlength{\tabcolsep}{2pt}
    \renewcommand{\arraystretch}{1.0}
    \begin{tabular}{lc}
        \toprule
        \textbf{Parameter} & \textbf{Value} \\
        \midrule
        Maximum new tokens & $16{,}384$ \\
        Thinking mode & Enabled \\
        Temperature & $0.6$ \\
        Top-$p$ & $0.95$ \\
        Top-$k$ & $20$ \\
        Min-$p$ & $0.0$ \\
        Samples per prompt & $12$ \\
        Sampling seed & $0$ \\
        \bottomrule
    \end{tabular}
    
\end{table}

\section{Notation Summary}
\label{app:notation}

Table~\ref{tab:notation_summary} summarizes the notation used throughout the OPASD formulation.

\begin{table}[t]
\centering
\caption{Notation summary. Symbols used in the formulation and theoretical analysis of OPASD.}
\label{tab:notation_summary}
\small
\setlength{\tabcolsep}{2pt}
\renewcommand{\arraystretch}{1.0}

\begin{tabular}{
    p{0.18\linewidth}
    p{0.65\linewidth}
    p{0.10\linewidth}
}
\toprule
\textbf{Symbol} & \textbf{Meaning} & \textbf{First Use} \\
\midrule

\multicolumn{3}{l}{\emph{Data, models, and trajectories}} \\

$\mathcal{S}$, $N$
& Reasoning dataset and number of problem--solution pairs
& Sec.~\ref{sec:opsd} \\

$x_i$, $y_i^\star$
& Problem and its verified reference solution
& Sec.~\ref{sec:opsd} \\

$p_\theta$, $p_{\theta_0}$
& Student model with trainable parameters $\theta$; frozen initial model used by the teacher
& Sec.~\ref{sec:opsd} \\

$p_S$, $p_T$
& Student and privileged teacher policies
& Eq.~(\ref{eq:eq1}) \\

$\hat{y}$, $T$
& Student-generated reasoning trajectory and its length
& Eq.~(\ref{eq:eq2}) \\

$t$, $\hat{y}_{<t}$
& Reasoning step and preceding trajectory prefix
& Eq.~(\ref{eq:eq3}) \\

$\mathcal{B}$
& Training minibatch sampled from $\mathcal{S}$
& Alg.~\ref{alg:opasd} \\

\midrule
\multicolumn{3}{l}{\emph{Token and attention distributions}} \\

$p_t^S$, $p_t^T$
& Student and teacher next-token distributions at reasoning step $t$
& Eq.~(\ref{eq:eq3}) \\

$N_h$, $h$
& Number of attention heads and attention-head index
& Eq.~\eqref{eq:attention_aggregation} \\

$A^S$, $A^T$
& Head-averaged final-layer attention matrices of the student and teacher
& Eq.~\eqref{eq:attention_aggregation} \\

$Q_t^S$, $Q_t^T$
& Student and teacher query positions corresponding to token $\hat{y}_t$
& Sec.~~\ref{sec:support_mismatch} \\

$V_t^S$, $V_t^T$
& Student- and teacher-visible attention supports at reasoning step $t$
& Sec.~\ref{sec:support_mismatch} \\

$\mathcal{P}$
& Reference-solution positions available only to the teacher
& Sec.~~\ref{sec:support_mismatch} \\

$A_t^S$, $A_t^T$
& Student and teacher attention distributions obtained by restricting the head-averaged attention rows to their respective visible supports and renormalizing
& Sec.~~\ref{sec:support_mismatch} \\

$\widetilde{A}_t^T$
& Teacher attention projected and renormalized onto the student-visible support
& Eq.~\eqref{eq:projection} \\

\midrule
\multicolumn{3}{l}{\emph{Objectives and hyperparameters}} \\

$D_{\mathrm{tok}}$
& Token-distribution divergence
& Eq.~\eqref{eq:token_loss} \\

$D_{\mathrm{attn}}$
& Generalized Jensen--Shannon divergence between projected teacher attention and student attention
& Eq.~\eqref{eq:attention_loss} \\

$\beta$, $M_t^{(\beta)}$
& JSD interpolation coefficient and corresponding mixture distribution
& Eq.~\eqref{eq:attention_loss} \\

$\ell_{\mathrm{attn}}(t)$
& Per-step attention-distillation loss
& Eq.~\eqref{eq:attention_loss} \\

$\mathcal{L}_{\mathrm{tok}}$
& Token-level distillation objective
& Eq.~\eqref{eq:token_loss} \\

$\mathcal{L}_{\mathrm{attn}}$
& Attention-level distillation objective
& Eq.~\eqref{eq:attention_loss_total} \\

$\lambda_{\mathrm{attn}}$
& Attention-loss weight
& Eq.~\eqref{eq:total_loss} \\

$\mathcal{L}$
& Final joint-distillation objective
& Eq.~\eqref{eq:total_loss} \\

\midrule
\multicolumn{3}{l}{\emph{Theoretical analysis}} \\

$\rho_t$
& Teacher attention mass assigned to reference-solution positions
& App.~\ref{app:prelim} \\

$D_\beta$
& Generalized Jensen--Shannon divergence used in the theoretical analysis
& App.~\ref{app:prelim} \\

$\alpha_t$
& Effective JSD interpolation coefficient on the student-visible support
& App.~\ref{app:projection} \\

$C_\beta(\rho)$
& Minimum divergence determined by the teacher attention mass outside the student-visible support
& App.~\ref{app:projection} \\

\bottomrule
\end{tabular}

\end{table}

\section{Theoretical Analysis}
\label{app:theory}

This appendix provides a theoretical characterization of OPASD's
student-support projection. We show that, among all attention distributions
supported on positions visible to the student, the projected teacher
attention uniquely minimizes the generalized Jensen--Shannon divergence
to the full privileged-teacher attention distribution.

\subsection{Preliminaries}
\label{app:prelim}

\paragraph{Attention distributions.}
For a fixed problem $x$, reference solution $y^\star$, sampled trajectory
$\hat{y}$, and reasoning step $t$, let $V_t^S$ and $V_t^T$ denote the
student and teacher attention supports defined in
Sec.~\ref{sec:support_mismatch}. The reference-solution positions
available only to the teacher are denoted by $\mathcal{P}$, with
\[
V_t^T = V_t^S \cup \mathcal{P},
\qquad
V_t^S \cap \mathcal{P} = \varnothing.
\]


Let $A_t^T$ denote the teacher attention of Sec.~\ref{sec:support_mismatch},
renormalized over $V_t^T$; this leaves Eq.~\eqref{eq:projection} unchanged. We denote the teacher attention mass assigned to the
reference-solution positions by
$\rho_t=\sum_{i\in\mathcal{P}}A_t^T(i)$. The remaining mass on
$V_t^S$ is therefore $1-\rho_t$, and the student-support projection
in Eq.~\eqref{eq:projection} can be written as
\[
1-\rho_t
=
\sum_{i\in V_t^S}A_t^T(i),
\qquad
\widetilde{A}_t^T(i)
=
\frac{A_t^T(i)}{1-\rho_t},
\quad i\in V_t^S.
\]
For the nonempty student-visible support used in OPASD, the denominator
in Eq.~\eqref{eq:projection} is positive under softmax attention, so
$\rho_t<1$ and the projection is well defined. Equivalently,
$\widetilde{A}_t^T$ is the teacher attention conditioned on the attended
position lying in $V_t^S$. As shown in
Eq.~\eqref{eq:ratio_preservation}, this conditioning preserves the
teacher's relative attention between student-visible positions.

\paragraph{Attention divergence.}
For the analysis below, we write $D_\beta(p,q)$ for the generalized
Jensen--Shannon divergence used in Eq.~\eqref{eq:attention_loss},
\[
D_\beta(p,q)
=
\beta D_{\mathrm{KL}}(p\Vert M)
+
(1-\beta)D_{\mathrm{KL}}(q\Vert M),
\qquad
M=\beta p+(1-\beta)q,
\]
where $\beta\in(0,1)$ and $p$ and $q$ are probability distributions
on the same finite set. We use the standard convention
$0\log(0/a)=0$ for $a\ge0$. The divergence is finite because
$M\ge\beta p$ and $M\ge(1-\beta)q$ entrywise. At
$\beta=\frac12$, $D_\beta$ reduces to the Jensen--Shannon
divergence~\citep{lin1991divergence}. The teacher distribution appears
as the first argument, matching the ordering used in
Eq.~\eqref{eq:attention_loss}.

\begin{lemma}[Nonnegativity of $D_\beta$]
\label{lem:dbeta}
For $\beta\in(0,1)$, $D_\beta(p,q)\ge0$, with equality if and only if
$p=q$.
\end{lemma}

\begin{proof}
By Gibbs' inequality, both KL terms in $D_\beta(p,q)$ are nonnegative
and equal zero only when their two arguments are identical. Since
$\beta\in(0,1)$, $D_\beta(p,q)=0$ requires
$p=M$ and $q=M$, where $M=\beta p+(1-\beta)q$. Hence $p=q$.
Conversely, $p=q$ gives $D_\beta(p,q)=0$.
\end{proof}

\subsection{Optimality of Student-Support Projection}
\label{app:projection}

We now compare the projected teacher attention with any other attention
distribution $q$ supported on $V_t^S$. When comparing $q$ with $A_t^T$
over $V_t^T$, we take $q(i)=0$ for $i\in\mathcal{P}$. Define
\[
\alpha_t
=
\frac{\beta(1-\rho_t)}{1-\beta\rho_t},
\qquad
C_\beta(\rho)
=
\beta\rho\log\frac{1}{\beta}
+
\beta(1-\rho)\log\frac{1-\rho}{1-\beta\rho}
-
(1-\beta)\log(1-\beta\rho).
\]

\begin{theorem}[Optimality of the projected target]
\label{thm:projection}
For $\beta\in(0,1)$ and any attention distribution $q$ supported on
$V_t^S$,
\begin{equation}
D_\beta(A_t^T,q)
=
C_\beta(\rho_t)
+
(1-\beta\rho_t)
D_{\alpha_t}(\widetilde{A}_t^T,q).
\label{eq:app_jsd_projection}
\end{equation}
Hence $\widetilde{A}_t^T$ is the unique minimizer of
$D_\beta(A_t^T,q)$ over all attention distributions supported on
$V_t^S$. The minimum value $C_\beta(\rho_t)$ is zero if and only if
$\rho_t=0$ and is strictly increasing in $\rho_t$.
\end{theorem}

\begin{proof}
Write $s=1-\rho_t$ and $Z=1-\beta\rho_t$. Since $\rho_t<1$ and
$\beta\in(0,1)$, both are positive. Moreover,
\[
\alpha_t=\frac{\beta s}{Z},
\qquad
1-\alpha_t=\frac{1-\beta}{Z},
\]
and
$Z-\beta s=1-\beta>0$. Hence $0<\alpha_t<1$.

Let $M=\beta A_t^T+(1-\beta)q$ be the mixture in
$D_\beta(A_t^T,q)$. Since $q$ assigns zero mass to $\mathcal{P}$,
$M=\beta A_t^T$ on $\mathcal{P}$. On $V_t^S$, we have
$A_t^T=s\widetilde{A}_t^T$. Defining
\[
R
=
\alpha_t\widetilde{A}_t^T+(1-\alpha_t)q,
\]
we obtain
\[
M
=
\beta s\widetilde{A}_t^T+(1-\beta)q
=
ZR
\qquad\text{on }V_t^S.
\]
Since
$R\ge\alpha_t\widetilde{A}_t^T$ and
$R\ge(1-\alpha_t)q$ entrywise, the KL terms below are finite.

Splitting the teacher KL term over $\mathcal{P}$ and $V_t^S$ gives
\begin{align*}
\beta D_{\mathrm{KL}}(A_t^T\Vert M)
&=
\beta\rho_t\log\frac{1}{\beta}
+
\beta s
\sum_{i\in V_t^S}
\widetilde{A}_t^T(i)
\log
\frac{s\widetilde{A}_t^T(i)}{ZR(i)}
\\
&=
\beta\rho_t\log\frac{1}{\beta}
+
\beta s\log\frac{s}{Z}
+
\beta s
D_{\mathrm{KL}}(\widetilde{A}_t^T\Vert R).
\end{align*}
Since $q$ assigns zero mass to $\mathcal{P}$, the second KL term is
\[
(1-\beta)D_{\mathrm{KL}}(q\Vert M)
=
-(1-\beta)\log Z
+
(1-\beta)D_{\mathrm{KL}}(q\Vert R).
\]
The terms independent of $R$ are exactly $C_\beta(\rho_t)$. Using
$\beta s=Z\alpha_t$ and $1-\beta=Z(1-\alpha_t)$, the remaining terms
give
\begin{align*}
D_\beta(A_t^T,q)
&=
C_\beta(\rho_t)
+
Z\Bigl[
\alpha_t
D_{\mathrm{KL}}(\widetilde{A}_t^T\Vert R)
+
(1-\alpha_t)
D_{\mathrm{KL}}(q\Vert R)
\Bigr]
\\
&=
C_\beta(\rho_t)
+
(1-\beta\rho_t)
D_{\alpha_t}(\widetilde{A}_t^T,q),
\end{align*}
which proves Eq.~\eqref{eq:app_jsd_projection}.

By Lemma~\ref{lem:dbeta},
$D_{\alpha_t}(\widetilde{A}_t^T,q)\ge0$, with equality if and only if
$q=\widetilde{A}_t^T$. Since $1-\beta\rho_t>0$, the projected teacher
attention is the unique minimizer.

Finally, $C_\beta(0)=0$. Differentiating gives
\[
C_\beta'(\rho)
=
\beta\log\frac{1-\beta\rho}{\beta(1-\rho)},
\]
where the non-logarithmic terms cancel because
$\beta(1-\rho)+(1-\beta)=1-\beta\rho$.
For $0\le\rho<1$,
$1-\beta\rho>\beta(1-\rho)$ because their difference is
$1-\beta>0$. Hence $C_\beta'(\rho)>0$, so $C_\beta(\rho)$ is zero
if and only if $\rho=0$ and is strictly increasing in $\rho$.
\end{proof}

Theorem~\ref{thm:projection} provides a theoretical justification for
the student-support projection. Among all attention distributions
supported on the student-visible positions, $\widetilde{A}_t^T$
uniquely minimizes the generalized Jensen--Shannon divergence to the
full privileged-teacher attention, while the minimum divergence
increases with the teacher attention mass assigned to positions
unavailable to the student.

\end{document}